\documentclass{article}
\usepackage{ijcai26}

\usepackage{times}
\usepackage{soul}
\usepackage{url}
\usepackage[hidelinks]{hyperref}
\usepackage[utf8]{inputenc}
\usepackage[small]{caption}
\usepackage{graphicx}
\usepackage{amsmath}
\usepackage{amsthm}
\usepackage{booktabs}
\usepackage{algorithm}
\usepackage{algorithmic}
\usepackage[switch]{lineno}

\usepackage{amsfonts,amssymb,mathtools,xspace}
\usepackage{thmtools} 
\usepackage{xcolor}
\usepackage[algo2e,ruled,vlined,linesnumbered]{algorithm2e}\SetArgSty{upshape}
\usepackage[nameinlink,capitalise]{cleveref}

\newtheorem{theorem}{Theorem}

\newtheorem{lemma}[theorem]{Lemma}

\title{Speeding Up the NSGA-II via Dynamic Population Sizes}

\author{
Benjamin Doerr$^1$
\and
Martin~S. Krejca$^1$\and
Simon Wietheger$^2$\\
\affiliations
$^1$Laboratoire d'Informatique (LIX), CNRS, École Polytechnique, Institut Polytechnique de Paris\\
$^2$Algorithms and Complexity Group, TU Wien\\
\emails
\{firstname.lastname\}@polytechnique.edu,
swietheger@ac.tuwien.ac.at,
}

\DeclareRobustCommand{\NSGAtwo}{NSGA\nobreakdash-II\xspace}
\DeclareRobustCommand{\NSGAthree}{NSGA\nobreakdash-III\xspace}
\DeclareRobustCommand{\SMS}{SMS\nobreakdash-EMOA\xspace}
\newcommand{\SPEA}{SPEA2\xspace}

\newcommand{\jump}{\textsc{Jump}\xspace}
\newcommand{\zerojump}{\textsc{Zero\-Jump}\xspace}
\newcommand{\omm}{\textsc{OMM}\xspace}

\newcommand{\lotz}{\textsc{LOTZ}\xspace}
\newcommand{\ojzj}{\textsc{OJZJ}\xspace}
\newcommand{\ojzjk}{\ensuremath{{\textsc{OJZJ}_k}}\xspace}

\newcommand{\Nmax}{N_{\max}}

\newcommand{\R}{\ensuremath{\mathbb{R}}}

\newcommand{\N}{\ensuremath{\mathbb{N}}}

\newcommand{\EE}{\ensuremath{\mathcal{E}}}

\newcommand{\eps}{\varepsilon}

\let\originalleft\left
\let\originalright\right
\renewcommand{\left}{\mathopen{}\mathclose\bgroup\originalleft}
\renewcommand{\right}{\aftergroup\egroup\originalright}

\newcommand{\set}[1]{\{#1\}}

\definecolor{orange}{RGB}{255,127,0}

\newcommand{\nootherflip}[1][n-1]{\left(1-\tfrac{1}{n}\right)^{#1}}

\DeclarePairedDelimiter\floor{\lfloor}{\rfloor}
\DeclarePairedDelimiter\ceil{\lceil}{\rceil}

\let\oldsqrt\sqrt
\def\hksqrt{\mathpalette\DHLhksqrt}
\def\DHLhksqrt#1#2{\setbox0=\hbox{$#1\oldsqrt{#2\,}$}\dimen0=\ht0
   \advance\dimen0-0.2\ht0
   \setbox2=\hbox{\vrule height\ht0 depth -\dimen0}%
   {\box0\lower0.4pt\box2}}
\renewcommand\sqrt\hksqrt

\DeclareMathOperator{\MEI}{MEI}
\DeclareMathOperator{\MEIE}{MEIE}
\DeclareMathOperator{\EIE}{EIE}

\DeclareMathOperator{\phasel}{Phase}

\newcommand{\ones}[1]{\ensuremath{|#1|_1}}

\usepackage{makecell, multirow}

\begin{document}

\maketitle

\begin{abstract}
    Multi-objective evolutionary algorithms (MOEAs) are among the most widely and successfully applied optimizers for multi-objective problems. However, to store many optimal trade-offs (the {Pareto optima}) simultaneously, MOEAs are typically run with a large population of solution candidates. This slows down the algorithm and renders the choice of the population size a crucial design decision. In this work, we aim to overcome these difficulties by proposing the \emph{dynamic \NSGAtwo}, a variant of the well-known \NSGAtwo that starts with a small initial population and doubles it after a user-specified number~$\tau$ of function evaluations, up to a maximum size of~$N_{\max}$. We prove that the dynamic \NSGAtwo with optimal parameters computes the Pareto front of the \textsc{OneMinMax} benchmark of size~$n$ with high probability in $O(n \log^2 n)$ function evaluations, which is considerably faster than the $\Theta(n^2 \log n)$ runtime of the static \mbox{\NSGAtwo} with optimal parameters. For the \textsc{OneJumpZeroJump} benchmark with gap size~$k$, we show a runtime of $O(n^k \log^2 n)$, improving upon the known runtime of $\Theta(n^{k+1})$. We also propose a variant that uses the initial population size for a longer period and achieves slightly better performance. Finally, we show that a simple concurrent-run strategy turns our dynamic \mbox{\NSGAtwo} variants into parameter-less algorithms that  exceed the above runtimes only by a logarithmic factor and hence still outperform the static \NSGAtwo by a factor of~$\tilde\Omega(n)$.
\end{abstract}

\section{Introduction}

Real-world problems often require the optimization of conflicting objectives, resulting in several incomparable optimal trade-offs, known as \emph{Pareto optima}. One solution concept for such multi-objective optimization problems is to compute several Pareto optima, ideally all of them (the \emph{Pareto front}), and let a decision maker select the final solution.
\emph{Multi-objective evolutionary algorithms (MOEAs)} are a natural choice for this task as they usually work with sets of solutions (\emph{populations}) and, in fact, MOEAs are among the most successful and widely applied approaches in multi-objective optimization \cite{CoelloLV07,ZhouQLZSZ11}.

Recently, the empirical success of popular state-of-the-art MOEAs, such as the \NSGAtwo \cite{DebPAM02}, NSGA-III \cite{DebJ14}, SMS-EMOA \cite{BeumeNE07}, MOEA/D \cite{ZhangL07}, or SPEA2 \cite{ZitzlerLT01}, has been supported by the first mathematical analyses of these algorithms \cite{ZhengLD22,WiethegerD23,BianZLQ23,LiZZZ16,RenBLQ24}.
This not only gave rigorous performance guarantees for MOEAs but also led to deeper insights into their merits and shortcomings.
As a result, theoretically motivated and provably improved MOEAs have been presented, e.g., in \cite{BianZLQ23,DoerrIK25}. We aim at this as well, for the most prominent MOEA, the \NSGAtwo.

When analyzing the mathematical runtime analyses of the \NSGAtwo, we observe an ambivalent role of the population size $N$, the main parameter of this algorithm. On the one hand, the population size must be large enough to represent the Pareto front (and, in fact, a little larger, since the algorithm fails to perfectly efficiently cover the Pareto front) \cite{ZhengD23aij}. On the other hand, since the \NSGAtwo generates and evaluates $N$ solutions per iteration, a large~$N$ slows down the algorithm. This effect was made precise in \cite{DoerrQ23LB}. These two insights, together with the details of the proofs in these works, suggest to design an \NSGAtwo that gradually increases the population size.

\textbf{Our contribution.}
We propose the \emph{dynamic \NSGAtwo} (\Cref{alg:nsga-ii-regular-growth}), which is an \NSGAtwo that starts with a small population of size four and doubles it after a user-specified number $\tau$ of function evaluations, but not beyond a maximum population size $\Nmax$. This second main parameter of the algorithm plays a similar role as the population size $N$ of the classic \NSGAtwo, in particular, it has to be large enough for the algorithm to eventually cover the Pareto front. Since both empirical \cite{KukkonenD06} and mathematical \cite{ZhengD25approx} results have shown that it is preferable to select the next population of the \NSGAtwo sequentially, always based on the current crowding distance at that moment, we use this variant of the \NSGAtwo.

We analyze mathematically the performance of this dynamic \NSGAtwo on two standard benchmarks from the theory of MOEAs. For the OneMinMax benchmark (denoted here by \omm), we show that our algorithm with optimal parameters with high probability finds the Pareto front within $O(n \log^2 n)$ function evaluations (\Cref{thm:uniform_growth}). This is considerably faster than the $O(n^2 \log n)$ best runtime guarantees known for the standard versions of the \NSGAtwo, \mbox{\NSGAthree}, \SMS, \SPEA, and MOEA/D with optimal parameters \cite{ZhengD23aij,DangOSS23gecco,ZhengD24,RenBLQ24,LiZZZ16}, proven to be asymptotically tight for the \NSGAtwo \cite{DoerrQ23LB}. For the OneJumpZeroJump benchmark (\ojzjk) with difficulty parameter $k \in [2 .. \floor{\frac{n}{2}}]$, we prove that the dynamic \NSGAtwo with optimal parameters satisfies a runtime guarantee of $O(n^k \log^2 n)$ function evaluations with high probability (\Cref{thm:uniform_growth_ojzj}).
Except for very large values of $k$, this compares again favorably with the best known runtime guarantee of $O((n - 2k) n^k)$ for the \NSGAtwo, \mbox{\NSGAthree},  \SMS, and \SPEA \cite{DoerrQ23tec,Opris25ijcai,BianZLQ23,RenBLQ24}, which is again tight for the classic \NSGAtwo \cite{DoerrQ23LB}; we are not aware of any runtime analysis of the MOEA/D on \ojzj. Both of our runtime guarantees are independent of the maximum population size $\Nmax$ as long as this is at least slightly larger than the size of the Pareto front.

Noting that the optimal choice of the doubling parameter~$\tau$ is mostly determined by the part of the optimization process working with the initial population size, we also propose a variant of the dynamic \NSGAtwo that allows for a longer initial phase. For this \emph{dynamic \NSGAtwo with longer initial phase}, we set the length of the initial phase in such a manner that it takes as long as all later phases combined until the maximum population size is reached. This significantly increases the time spent with a small population without increasing the asymptotic order of the runtime. In our runtime analysis, we show that this variant allows for a smaller doubling factor (for all but this initial phase), resulting in slightly superior runtimes (\Cref{thm:omm-extended-first-phase,thm:ojzj-extended-first-phase}), e.g., an $O(n \log n)$ runtime on \omm, which is asymptotically optimal in a large class of unbiased black-box algorithms \cite{LehreW12}.

Our results described so far stated the performance for an optimal choice of the doubling parameter~$\tau$. They display a linear dependence on $\tau$ when $\tau$ is at least of a certain size. Below this, the runtime essentially reverts to the one of the classic \NSGAtwo with population size~$\Nmax$. In that sense, a suboptimal choice of $\tau$ is not fatal, but the advantage of the dynamic choice of the population size could decrease or vanish. To relieve the user from manually optimizing $\tau$, we propose a simple concurrent-run strategy (\Cref{alg:concurrent_opt}) that removes the parameter $\tau$ from the algorithm at the price of a runtime that is by a logarithmic factor larger than the runtime obtainable with the optimal choice of~$\tau$ (\Cref{thm:concurrent_omm_ojzj}). This strategy works both for the dynamic \NSGAtwo and its variant with longer initial phase. Since the runtime of the former was independent of $\Nmax$, we obtain here a version of the \NSGAtwo that has no parameters related to the population size.

\begin{table*}
    \centering
    \begin{tabular}{*{3}{l}}
        Algorithm       & \omm                                                                        & \ojzjk ($k \in [2 .. \floor{\frac{n}{2}}]$)                    \\
        \toprule
        static          & $\Theta(N n \log n)$                                                        & $\Theta(N n^k)$                                                \\
                        & $\quad N \geq 4(n+1)$
                        & $\quad N \geq 4(n-2k+3)$                                                                                                                     \\
        \midrule
        dynamic         & $O(\tau \log n)$ (\Cref{thm:uniform_growth})                                & $O(\tau \log n)$ (\Cref{thm:uniform_growth_ojzj})              \\
                        & $\quad\tau \geq C n \ln(n)$                                                 & $\quad\tau \geq C n^k \ln(n)$                                  \\
        \midrule
        \multirow{2}{*}{\makecell[tl]{dynamic with                                                                                                                     \\longer
        initial phase}} & $O(\tau \log( N_{\max})+\Nmax\log n)$ (\Cref{thm:omm-extended-first-phase}) & $O(\tau \log N_{\max})$ (\Cref{thm:ojzj-extended-first-phase}) \\
                        & $\quad\tau \geq C n$                                                        & $\quad\tau \geq C n^k$
    \end{tabular}
    \caption{
        Number of function evaluations for the static \NSGAtwo (\Cref{alg:nsga-ii-classic}), the dynamic \NSGAtwo (\Cref{alg:nsga-ii-regular-growth}), and the dynamic \NSGAtwo with longer initial phase until the Pareto front is covered.
        The results for the static \NSGAtwo on \omm are due to {\protect\cite{ZhengD23aij,DoerrQ23LB}}, and on \ojzjk due to {\protect\cite{DoerrQ23tec,DoerrQ23LB}}, and they state an expected value.
        The other results hold with probability $1 - O(\frac{1}{n})$ and assume that $C \in \R_{> 0}$ is a sufficiently large constant and that $N_{\max} \geq n + 4$.
    }
    \label{tab:results}
\end{table*}

\section{Related Work}

Doubling schemes have been considered for single-objective evolutionary optimization, e.g., \cite{HarikL99,DoerrD20bookchapter}.
To the best of our knowledge,
this work is the first to study a MOEA with dynamic parameters via mathematical means, and our findings heavily rely on the multi-objective nature of the considered problems.
In such, there is little directly related work to be discussed here.

What comes closest to our work are the recent works that study how to run a MOEA with a smaller population size with the help of an archive. An archive is a storage for good solution candidates that do not take part in the evolutionary process.
This way, good solutions can be stored outside the main population of the MOEA, possibly allowing the MOEA to be run with a smaller population size. We refer to \cite{LiLY24} for a recent general survey on this technique and review the two independent theoretical works on this topic.

\cite{BianRLQ24} consider the \NSGAtwo with an archive. They prove that this algorithm optimizes the $\omm$ and $\lotz$ benchmark in $O(n\log n)$ and $O(n^2)$ function evaluations, respectively, when a constant population size is chosen.
To achieve these runtimes, the analysis relies on an intensive use  of one-point crossover. This is natural for a problem like \lotz, but less natural for a problem like \omm, in which there is no particular order of the bits.
We are very skeptical that the \NSGAtwo with archive and constant population size can optimize the \omm or \ojzjk benchmarks without this type of crossover.
All results proven in this work hold if each parent in each iteration has a constant chance to receive mutation only, allowing arbitrary use of crossover otherwise.

\cite{DoerrKNN24} consider the MOEA/D, which by definition uses an archive. Their main result is that the mutation-only MOEA/D with a population size smaller than the Pareto front size has enormous difficulties finding the Pareto front of \omm when bit-wise mutation is used. However, when a heavy-tailed mutation operator (having a higher sampling variance) is used, smaller population sizes are admissible and lead to considerably better runtimes. More precisely, for $N = O(n^{\beta-1})$ the runtime becomes $O(n^\beta \log n)$, where the constant $\beta>1$ is the exponent of the power-law distribution used to define the mutation operator. By choosing $\beta$ small, runtimes of $O(n^{1+\eps})$ can be obtained. 
Building on this work, ~\cite{CorusOY25} showed that the MOEA/D with Fast Contiguous Somatic Hypermutation can optimize \omm in a runtime of order $O(n \log n)$.
The MOEA/D with constant population size and one-point crossover was shown to optimize \omm in time $O(n \log n)$ in~\cite{HuangZCHLX21}.

From these results and our understanding of the proofs, we tend to believe that working with a growing population size is slightly preferable to an archive. Since the solutions in the archive can never be used to generate new solutions, there is a considerable risk (as proven in the bit-wise mutation setting regarded in \cite{DoerrKW24}) that the few solutions in the population are not sufficient to reach all solutions on the Pareto front. That said, a dynamic population size requires choosing the right growth schedule for the population.

\section{Preliminaries}

Let~$\N$ denote the positive natural numbers, and $\N_0 \coloneqq \set{0} \cup \N$. For $i, j\in \N_0$, let $[i .. j] \coloneqq [i, j] \cap \N_0$, $[i] \coloneqq [1 .. i] $, and $[i]_0 \coloneqq [0 .. i]$. We abbreviate $\log_2(x)$ to $\log(x)$ and $\log_e(x)$ to $\ln(x)$.

We consider pseudo-Boolean \emph{maximization}, that is, for a given $n \in \N$, we study functions $f\colon \{0, 1\}^n \to \R^2$.
We call each $x \in \{0, 1\}^n$ an \emph{individual} and $f(x) \eqqcolon \bigl(f_1(x), f_2(x)\bigr)$ its \emph{objective value}.
Furthermore, for all $x \in \{0, 1\}^n$, we denote the number of ones of~$x$ by $\ones{x}$.

We compare objective values $u, v \in \R^2$ via the \emph{domination} partial order~$\succeq$.
We say \emph{$u$ weakly dominates~$v$} ($u \succeq v$) if $u_1 \geq v_1$ and $u_2 \geq v_2$; and~$u$ and~$v$ are \emph{incomparable} if nei\-ther $u \succeq v$ nor $v\succeq u$.
If and only if $u \succeq v$ and $u \neq v$, we say \emph{$u$ strictly dominates~$v$} ($u \succ v$).
We use domination also for individuals, referring to domination of their objective values.

For a function~$f$, we call each $x \in \{0, 1\}^n$ \emph{Pareto-optimal} if and only if no $y \in \{0, 1\}^n$ satisfies $f(y) \succ f(x)$.
For each Pareto-optimal $x \in \{0, 1\}^n$, we call~$f(x)$ a \emph{Pareto optimum}.
The \emph{Pareto front} is the set of all Pareto optima.

\subsubsection{The \textsc{OneMinMax} Benchmark}

We consider the \textsc{OneMinMax} (\omm) benchmark function \cite{GielL10}, which is the most commonly studied benchmark in the theory of MOEAs.
\omm features two conflicting objectives: the first returns the number of zeros in a given individual, the second its number of ones.
Formally, we have $\omm\colon \{0, 1\}^n \to [n]_0^2$ with $x \mapsto (n - \ones{x}, \ones{x})$.

Due to the conflicting objectives, all individuals are Pareto-optimal, and the Pareto front is $\{(i, n - i) \mid i \in [n]_0\}$.

\subsubsection{The \textsc{OneJumpZeroJump} Benchmark}

As a second benchmark function, we consider \textsc{OneJumpZeroJump} (\ojzjk) \cite{DoerrZ21aaai}, which is based on the commonly studied single-objective benchmark \jump \cite{DrosteJW02} benchmark.
Similar to \omm, \ojzjk features two conflicting objectives that try to maximize the number of ones and zeros, respectively. However, given a hardness parameter $k \in [\floor{\frac{n}{2}}]$, both objectives have low values for numbers of ones (zeros) which are at a distance of~$1$ to~$k$ to the maximum number~$n$, requiring the population to jump across this valley of small objective values.
Formally, we have $\ojzjk\colon \{0, 1\}^n \to [n+k]^2_0$ with $x \mapsto (\jump_k(x), \zerojump_k(x))$, where
\begin{align*}
    \jump_k(x) = \begin{cases}
                     |x|_1 +k & \text{if } |x|_1 \le n-k \text{ or } |x|_1 = n; \\
                     n-|x|_1  & \text{else;}
                 \end{cases}
\end{align*}
and $\zerojump_k$ is like $\jump_k$, swapping ones and zeros.

Individuals in the valley of low objective values are not Pareto-optimal, and the Pareto front is $\{(n+k,0), (0,n+k)\}\cup \{(i+k, n - i+k) \mid i\in \N, k \le i \le n-k\}$. We call the objective values $(n+k,0)$ and $(0,n+k)$ the \emph{outer} Pareto front, and the remaining Pareto front the \emph{inner} Pareto front.

Throughout this work, when we mention \omm or \ojzjk, we assume that the problem size $n \in \N$ is given.
Our mathematical statements aim at asymptotics in terms of this~$n$, implying that we may assume~$n$ to be sufficiently large.

\subsubsection{Runtime}

During the run of an MOEA, we assume that the objective value of an individual is evaluated exactly once---at the point of its creation.
Hence, each potential copy of an individual is also evaluated exactly once.
The \emph{runtime} of an MOEA is the (random) number of function evaluations until the current multi-set of individuals that the MOEA maintains (its \emph{population}) \emph{covers} the Pareto front for the first time, that is, until the objective values of the population contain the Pareto front.

\subsubsection{Empty Intervals}

Given a population~$P$ of individuals not covering the Pareto front of the chosen benchmark, we are interested in the consecutive objective values of the Pareto front that are not in~$P$.
We call these intervals \emph{empty}, defining them formally below.
The classic \NSGAtwo that uses the current crowding distance (details in Section~\ref{sec:static_nsgaII}) has the useful property that empty intervals (for \omm) do not increase over time and even reduce quickly in time, if the population size is large enough and if the algorithm's parent population contains the all-zeros bit string ($0^n$) as well as the all-ones bit string ($1^n$).
Note that the objective values of~$0^n$ and~$1^n$ are the two extreme points of the Pareto front of both \omm and \ojzjk.
Considering such populations is vital for our runtime analysis.

In the interest of clarity, we here only define empty intervals for \omm and defer the definition for \ojzjk to the respective section.
In order to allow the comparison of empty intervals among different populations, we follow the definition of Zheng and Doerr~\shortcite{ZhengD25approx} and define a fixed number of~$n$ empty intervals for any population, based on the points $\{i - 0.5 \mid i \in [n]\}$ in the space~$[0, n]$ of the first objective of \omm.
Formally, let~$P$ be a population of individuals containing at least~$0^n$ and~$1^n$.
For all $i \in [n]$, the \emph{$i$-th empty interval~$I_i$} (of~$P$ for~\omm) is defined as
\begin{align*}
    I_i = [\max & \set{f_1(x)\mid x\in P \land \omm_1(x)\le i-0.5}.. \\
    \min        & \set{f_1(x)\mid x\in P \land \omm_1(x)\ge i-0.5}].
\end{align*}
By definition, the objective values not in~$P$ are given by the interior points of the empty intervals.
For example, the empty interval $[0 .. 2]$ implies that the objective value $(1, n - 1)$ is not in the population. Different empty intervals can coincide.

Furthermore, given a population~$P$ of individuals containing at least~$0^n$ and~$1^n$, we define the \emph{maximum empty interval} (MEI) as the maximum length of its empty intervals, that is,
\begin{equation*}
    \MEI(P) = \max \set{|I_i| \mid i \in [n]},
\end{equation*}
where $|[a .. b]| = b - a$ denotes the length of an interval $[a .. b]$.
Note that if~$P$ covers the Pareto front, then $\MEI(P) = 1$.

\section{The Classic NSGA-II as a Baseline}\label{sec:static_nsgaII}

\begin{algorithm2e}[t]
    \caption{
        The classic (static) \NSGAtwo~\protect\cite{DebPAM02} with population size $N \in \N$ and selection based on current crowding distance \protect\cite{ZhengD25approx}, maximizing a given pseudo-Boolean function $f\colon \{0, 1\}^n \to \R^2$.
    }
    \label{alg:nsga-ii-classic}
    $t \coloneqq 0$\;
    $P_0 \coloneqq$ draw~$N$ individuals uniformly at random\;
    \While{termination criterion not met}{%
        $Q_t \coloneqq$ generate offspring population with size~$N$\;
        $P_{t + 1} \coloneqq$ select~$N$ individuals from $P_t \cup Q_t$ by non-dominated sorting followed by the current crowding distance (with uniform tie-breaker)\;
        $t \coloneqq t + 1$\;
    }
\end{algorithm2e}

The \emph{non-dominated sorting genetic algorithm~II} \cite{DebPAM02} (\NSGAtwo, \Cref{alg:nsga-ii-classic}) is the most popular and widely applied MOEA.
As such, we use it as a baseline for our dynamic versions.
However, we note that we use a slight improvement proposed by \cite{ZhengD25approx}, which applies the \emph{current crowding distance}, which handles ties in a manner that results in a more equidistant spread of solutions.

The \NSGAtwo maintains a multi-set of promising individuals (the \emph{parent population}) of a given size~$N$.
The initial population is generated uniformly at random and updated iteratively.
In each iteration, the \NSGAtwo creates an \emph{offspring population} of size~$N$ by user-defined means and then selects from the combined parent and offspring population the~$N$ most promising individuals for the next iteration.
This selection mechanism ranks individuals hierarchically: first, individuals are selected by their rank assigned via non-dominated sorting.
Ties are handled via the (current) crowding distance, and any remaining ties are handled uniformly at random.

\textbf{Computing the offspring population.}
The \NSGAtwo creates in total~$N$ offspring by repeatedly picking a certain number of individuals from the parent population and then creating modifications of them.
If it picks a single parent, the process of creating an offspring is called \emph{mutation}.

We allow for a variety of ways to select parents and to create offspring.
The only restriction we make is that there is a constant $q \in (0, 1]$ such that for each possible parent population~$P$ and each parent $x \in P$, the probability to select~$x$ at least once during the computation of the offspring population and to apply \emph{bit-wise mutation} to~$x$ is at least~$q$.

Given a parent $x \in \{0, 1\}^n$, bit-wise mutation creates an offspring $y \in \{0, 1\}^n$ by first copying~$x$ and then flipping each bit of~$y$ independently with probability~$\frac{1}{n}$.

\textbf{Non-dominated sorting.}
Non-dominated sorting assigns each individual~$x$ in a given (combined) population~$R$ a \emph{rank} that determines by how many individuals~$x$ is strictly dominated.
Individuals with the same rank are grouped into the same \emph{non-dominated front}, resulting in a partition $(F_i)_{i \in [0 .. k - 1]}$ of~$R$ (of some size~$k$).

The \NSGAtwo always attempts to add entire fronts in increasing order to its next parent population, as the individuals per front are incomparable.
Once at least~$N$ individuals are chosen, the algorithm does not add any further fronts.
Formally, $i^* = \min \{i \in [0 .. k - 1] \mid \sum_{j \in [0 .. i]} F_i \geq N\}$ is the \emph{critical rank}.
The fronts up to $i^* - 1$ are always part of the next parent population.
If adding~$F_{i^*}$ results exactly in a population size of~$N$, the selection process for the next parent population terminates.
Otherwise, the \NSGAtwo selects the best $N - |\bigcup_{j \in [0 .. i^* - 1]} F_i|$ individuals from~$F_{i^*}$ based on their \emph{crowding distance}, where higher values are preferable.

\textbf{Crowding distance.}
The crowding distance of each individual $x \in F_{i^*}$ is the sum of the crowding distance of~$x$ \emph{per objective}.
For each objective $j \in \{1, 2\}$, let $(y_k)_{k \in [|F_{i^*}|]}$ denote the individuals of~$F_{i^*}$ sorted in increasing order with respect to objective~$j$.
If~$x$ is in the first or last position, its crowding distance for objective~$j$ is positive infinity.
Otherwise, its crowding distance for objective~$j$ is the normalized distance in objective value to its direct neighbors, that is, it is $\bigl(f_j(y_{k^* + 1}) - f_j(y_{k^* - 1})\bigr) / \bigl(f_j(y_{|F_{i^*}|}) - f_j(y_1)\bigr)$.

The \emph{current} crowding distance \cite{ZhengD25approx} is a modification of the classic crowding distance that aims at decreasing the maximum distances of the resulting population.
The current crowding distance achieves this by iteratively removing an individual~$x$ with the smallest (classic) crowding distance and then recomputing all (classic) crowding distances\footnote{Note that~$x$ has at most four neighbors in its objectives.}, until the desired population size is achieved.

\section{The Dynamic NSGA-II}

Based on the classic, static \NSGAtwo (\Cref{alg:nsga-ii-classic}), we propose the \emph{dynamic \NSGAtwo} (\Cref{alg:nsga-ii-regular-growth}), which starts with an initial population of size~$4$ and attempts to double it periodically after a user-defined number $\tau \in \N$ of function evaluations.
An attempt is successful if and only if the current population size is less than the user-defined maximum population size $N_{\max} \in \N$.
We call the consecutive function evaluations between attempts \emph{phases}, noting that phase~$0$ starts at the beginning of the algorithm run and ends at the end of the iteration in which the population size is attempted to be doubled for the first time.
Moreover, note that the number of phases is unbounded although the number of doublings is not.

\begin{algorithm2e}[t]
    \caption{
        The dynamic \NSGAtwo with maximum population size $N_{\max} \in \N$ and doubling parameter $\tau \in \N$, maximizing a given pseudo-Boolean function $f\colon \{0, 1\}^n \to \R^2$. Our only assumption on the offspring generation is that each parent has a probability of at least a constant $q \in (0, 1]$ to be selected for generating an offspring via bit-wise mutation.
    }
    \label{alg:nsga-ii-regular-growth}
    $t \coloneqq 0$, $N \coloneqq 4$, $w \coloneqq 0$\;
    $P_0 \coloneqq$ draw~$N$ individuals uniformly at random\;
    \While{termination criterion not met}{%
        $Q_t \coloneqq$ generate offspring population with size~$N$\;
        $w \coloneqq w + N$\;
        \lIf{$w \geq \tau$}{%
            $w \coloneqq 0$, $N \coloneqq \max \{2 N, N_{\max}\}$
        }
        $P_{t + 1} \coloneqq$ select~$N$ individuals from $P_t \cup Q_t$ by non-dominated sorting followed by the current crowding distance (with uniform tie-breaker)\;
        $t \coloneqq t + 1$\;
    }
\end{algorithm2e}

It is important that we attempt doublings after~$\tau$ function evaluations and not iterations.
If we double the population size after a fixed number of iterations, then, since the work done per iteration is proportional to the population size, relatively little work is done with small population sizes. This risks losing any advantage from a dynamic value for the population size.
By counting function evaluations, we ensure that each phase receives roughly the same computational budget.

\paragraph{Performance Guarantees for \omm}

We start by proving the following bound for the \omm benchmark. For the optimal value of $\tau$, our runtime guarantee is $O(n\log^2 n)$, which is considerably faster than the $\Theta(n^2 \log n)$ known for the classic \NSGAtwo with optimal parameters \cite{ZhengD23aij,DoerrQ23LB}. We note that all slower doubling schedules with $\tau = o(n^2)$ still lead to an asymptotic improvement over the static \NSGAtwo. Likewise, a too small value of~$\tau$ is not detrimental, but reverts to the $O(\Nmax n \log n)$ performance of the classic \NSGAtwo.

\begin{theorem}\label{thm:uniform_growth}
    Consider the dynamic \NSGAtwo (Algorithm~\ref{alg:nsga-ii-regular-growth}) optimizing $\omm$
    with $\Nmax \ge n+4$.
    Let~$T$ denote the number of iterations until the population covers the Pareto front, and let~$F$ denote the number of function evaluations in the first~$T$ iterations. If $\tau \ge 64en\ln(n)/q$, then with probability at least $1 - O(1/n)$, we have $T = O(\tau)$ and $F = O(\tau\log n)$.

    For smaller values of $\tau$, we have $F = O(\Nmax n \log n)$ both in expectation and with probability $1 - O(1/n)$, the same for the static \NSGAtwo with population size~$\Nmax$.
\end{theorem}

We note that we did not optimize the constants in \Cref{thm:uniform_growth}.
Remarkably, for sufficiently large~$\tau$, both the bounds on~$T$ and~$F$ do not depend on~$N_{\max}$, as the dynamic \NSGAtwo covers the entire Pareto front with high probability once its population size is in the order of the size of the Pareto front, that is, order~$n$.
This makes the dynamic \NSGAtwo very robust in terms of choosing~$N_{\max}$, in particular, since a population size of $n + 1$ is required in order to cover the entire Pareto front, and our lower bound of $n + 4$ almost matches this required minimum.
Moreover, since we cover all possible cases for~$\tau$, we see that the dynamic \NSGAtwo is also very robust in terms of choosing the doubling parameter, being better or as good as the static \NSGAtwo for any $\tau = O(n^2)$.

Moreover, we note that the arguments for small~$\tau$ in \Cref{thm:uniform_growth} also apply to the failure case with probability $O(1/n)$. Hence if the dynamic \NSGAtwo fails to display the desired behavior (which happens with probability $O(1/n)$), then the runtime reverts to the classic $O(\Nmax n \log n)$ number of function evaluations, both in expectation and with probability $1 - O(1/n)$. We note further that by raising the constant in the requirement on $\tau$, all failure probabilities can be reduced to $n^{-c}$ for any desired constant $c$. We omit the details.

We prove \Cref{thm:uniform_growth} in two parts.
First, we show that, with high probability, the dynamic \NSGAtwo finds the extreme solutions~$0^n$ and~$1^n$ during phase~$0$.
Afterward, we show that the MEI roughly halves, with high probability, in each phase.

For the first part, we apply a result by \cite{ZhengD25approx}, which holds in expectation, whereas we require one that holds with high probability. We obtain such a result by following essentially the analysis in~\cite{ZhengD25approx} but adding a tail bound for harmonic sums of geometric random variables from \cite[Theorem~1.10.35]{Doerr20bookchapter}. Since the proof holds also for the \NSGAtwo with classic crowding distance and for arbitrary population sizes (of at least four), we formulate the result in this general way.

\begin{lemma}\label{lem:nsgaII_omm_outwards}
    Consider the \NSGAtwo (\Cref{alg:nsga-ii-classic}) with arbitrary initialization,  with survival based on the classic or current crowding distance, and with arbitrary population size $N \ge 4$. Assume that in each iteration, each parent with probability at least $q \in (0, 1]$ is chosen to generate an offspring via bit-wise mutation.
    Let $T$ denote the number of iterations until in this algorithm, when optimizing \omm, the population contains $1^n$ and $0^n$.
    Then $T\le 2e n\ln(n) / q$ with probability at least $1-\frac{2}{n}$.
\end{lemma}

For the second part, we rely on the following lemma showing that the MEI quickly reduces once~$0^n$ and~$1^n$ are found. This result and its proof are similar to \cite[Lemmas~$11$ and~$12$]{ZhengD25approx}, except that we need a more detailed result and a more precise tail bound, since we apply this statement for each phase separately (and thus also for phases of few iterations, where the strong concentration behavior is weaker).

\begin{lemma}\label{lem:nsgaII_omm_inwards}
    Consider a run of the \NSGAtwo (\Cref{alg:nsga-ii-classic}), generating the next population in any way that has the property that any parent has a probability of at least $q \in (0, 1]$ to be selected a parent of a bit-wise mutation, and using the current crowding distance for the selection of the next population.
    Consider some iteration $t\in \N_0$ of this algorithm optimizing \omm. Let $P_t$ be arbitrary except that $\set{0^n,1^n}\subseteq P_t$.
    Let $m,m'\in \N$ such that $\max\set{\frac{2n}{N-3},1} \le m'$ and $\MEI(P_t) \le m$.
    Let $\delta \in \N$ such that $\delta \ge  4 e(m-m')/q$. Then $\Pr[\MEI(P_{t+\delta}) \le m'] \ge 1 - n \exp(-\frac{\delta q}{16e})$.
\end{lemma}

\paragraph{Performance Guarantees for \ojzjk}

\begin{theorem}\label{thm:uniform_growth_ojzj}
    Consider the dynamic \NSGAtwo (Algorithm~\ref{alg:nsga-ii-regular-growth}) optimizing $\ojzjk$
    with $k \in [2 .. \floor{\frac{n}{2}}]$ and $\Nmax \ge n+4$.
    Let~$T$ denote the number of iterations until the population covers the Pareto front, and let~$F$ denote the number of function evaluations in the first~$T$ iterations. If
    $\tau \ge 4e (n^{k} + 8n^2) \ln(n)/q$,
    then with probability at least $1 - O(1/n)$, we have $T = O(\tau)$ and $F = O(\tau\log n)$.

    For smaller values of $\tau$, we have $F = O(\Nmax n^k \log n)$ with probability $1 - O(1/n)$ and $E[F]=O(\Nmax n^k)$, both the same as for the static \NSGAtwo with population size~$\Nmax$.
\end{theorem}

For an optimal choice of~$\tau$, the resulting runtime guarantee of $O(n^k \log^2 n)$ is considerably faster than the $\Theta(n^{k + 1})$ bound for the static \NSGAtwo \cite{DoerrQ23tec,DoerrQ23LB}, and it remains faster for any such choice with $\tau = o(n^{k + 1} / \log n)$.
Moreover, a too small choice of~$\tau$ just results in the same bound as for the static \NSGAtwo.

We note that, in comparison to our result for \omm (\Cref{thm:uniform_growth}), our speed-up is slower by a factor of $O(\log n)$, which is a consequence of our results holding with high probability.
The final steps in the analysis of \ojzjk usually wait for a mutation to occur that flips exactly~$k$ bits.
In expectation, this takes $O(n^k)$ iterations (with each iteration requiring~$N_{\max}$ function evaluations).
However, since this event follows a geometric distribution, it is not strongly concentrated around its expectation and requires and additional factor of $O(\log n)$ in order to hold with high probability.

The details of our proof for \Cref{thm:uniform_growth_ojzj} are in the appendix.
However, the outline of our proof strategy is very similar to that for \omm.
The main difference is that we first wait until we find the optima of the outer Pareto front, which requires in either direction a mutation that flips exactly~$k$ bits at once.
This happens with high probability for the stated choice of~$\tau$.
During the next phase, we wait to reach the outermost optima of the \emph{inner} Pareto front.
This requires for either direction a mutation that flips at least~$k$ \emph{arbitrary} bits, which is easier than the first phase.
From then on, the problem essentially boils down to \omm.

\section{Using a Longer Initial Phase}

As the proofs of the results in the previous section show, the initial phase of the dynamic \NSGAtwo is crucial as here the extreme points of the Pareto front are found. To accommodate for this, we propose a dynamic \NSGAtwo with longer initial phase. For all but the initial phase, it works as the dynamic \NSGAtwo (\Cref{alg:nsga-ii-regular-growth}), doubling its population size after a fixed budget of $\tau \in \N$ fitness evaluations.
For the initial phase, we use a budget that is roughly equal to the budget of all following phases (until the maximum population size~$N_{\max}$ is reached), namely $\ceil{\log(N_{\max}/4)} \tau$ function evaluations.
By this, the dynamic \NSGAtwo with longer initial phase spends more time in the very efficient phase with population size four, but this without increasing the asymptotic runtime. As we shall see, this allows to use smaller values for~$\tau$, leading to better runtimes.

We prove runtime guarantees for the dynamic \NSGAtwo with longer initial phase on the $\omm$ and $\ojzjk$ benchmarks in \Cref{thm:omm-extended-first-phase,thm:ojzj-extended-first-phase}, respectively.
For optimal parameter values, these runtime guarantees are better by a factor of~$O(n)$ compared to the best known bounds for static \NSGAtwo, in particular, they are better by a factor of $O(\log n)$ than the guarantees for optimal parameters for the dynamic \NSGAtwo.
However, the runtime now always depends on~$N_{\max}$, due to the choice of the budget for phase~$0$.

\paragraph{Performance Guarantees for \omm}

For optimal parameters, we show a runtime of $O(n \log n)$ function evaluations, which is considerably faster than the $\Theta(n^2 \log n)$ runtime guarantee for the static \NSGAtwo.
We note that no {unbiased} mutation-only black-box algorithm can find the all-ones string faster than in $\Theta(n \log n)$ time \cite{LehreW12}, so it is hard to imagine that a general-purpose heuristic can solve the \omm problem asymptotically faster than this.

\begin{theorem}\label{thm:omm-extended-first-phase}
    Consider the dynamic \NSGAtwo with longer initial phase optimizing $\omm$
    with $\Nmax \ge n+4$.
    Let~$T$ denote the number of iterations until the population covers the Pareto front, and let~$F$ denote the number of function evaluations in the first~$T$ iterations. If $\tau \ge 16en/q$, then with probability at least $1 - O(1/n)$, we have
    \begin{align*}
        T = O(\tau \log N_{\max}) \ \textrm{and} \
        F = O(\tau \log N_{\max} + \Nmax \log n).
    \end{align*}
    For smaller values of $\tau$, we have $F = O(\Nmax n \log n)$ both in expectation and with probability $1 - O(1/n)$, the same for the static \NSGAtwo (\Cref{alg:nsga-ii-classic}) with population size~$\Nmax$.
  
\end{theorem}

The proof of \Cref{thm:omm-extended-first-phase} is similar to the one of \Cref{thm:uniform_growth}, noting that the budget for phase~$0$ is of the same order.
Afterward, the budget for the dynamic \NSGAtwo with longer initial phase is slightly smaller, which is why we apply \Cref{lem:nsgaII_omm_inwards} only until the MEI is $O(\log n)$.
We then wait until the algorithm attains a population size linear in~$n$ and show via \Cref{lem:nsgaII_omm_inwards} that the algorithm finds all remaining Pareto optima within the next $O(\log n)$ phases.

\paragraph{Performance Guarantees for \ojzjk}

For optimal parameters, we show a runtime of $O(n^k \log n)$.
For small values of~$k$, this is close to the general $\Omega(n^k)$ lower bound when relying solely on bit-wise mutation to flip exactly~$k$ bits.

\begin{theorem}\label{thm:ojzj-extended-first-phase}
    Consider the dynamic \NSGAtwo with longer initial phase optimizing $\ojzjk$ with $k \in [2 .. \floor{\frac{n}{2}}]$
    and $\Nmax \ge n+4$.
    Let~$T$ denote the number of iterations until the population covers the Pareto front, and let~$F$ denote the number of function evaluations in the first~$T$ iterations. If $\tau \ge 4e (2n^{k}+8n^2) /q$, then with probability at least $1 - O(1/n)$, we have
    \begin{align*}
        T = O(\tau \log N_{\max}) \quad \textrm{and} \quad
        F = O(\tau \log N_{\max}).
    \end{align*}
    For smaller values of $\tau$, we have $F = O(\Nmax n^k \log n)$ with probability $1 - O(1/n)$ and $E[F]=O(\Nmax n^k)$, both the same as for the static \NSGAtwo (\Cref{alg:nsga-ii-classic}) with population size~$\Nmax$.
\end{theorem}

The proof strategy for \Cref{thm:omm-extended-first-phase} is similar to the one for \Cref{thm:uniform_growth_ojzj}.
As the first phase is now longer by a factor of $\Theta(\log N_{\max}) \geq \Theta(\log n)$, a value of $\tau = \Theta(n^k)$ is sufficient to find with high probability the outer Pareto front.
Afterward, finding the outermost Pareto optima on the inner front may have a smaller budget than in the case with the dynamic \NSGAtwo, but since it is easier to go from the outer front to the inner front than vice versa, the slightly reduced budget does not affect the time to find the outermost solutions on the inner front.
Last, we deal essentially with an \omm instance.

\section{Choosing the Phase Length Automatically}

To relieve users from manually select a suiting value for the phase length~$\tau$ of the dynamic \NSGAtwo and the dynamic \NSGAtwo with longer initial phase, we propose a framework (\Cref{alg:concurrent_opt}) that concurrently runs multiple instances of the dynamic variant with distinct values for~$\tau$, namely all powers of two.
For each such $\tau = 2^i$, we consider an instance $A_i$ of the dynamic variant.
The framework repeatedly picks an instance $A_i$ and runs its next phase. If an instance was picked before, we continue with the process as we last left it.
Each time, we pick an instance~$A_i$ such that the total number of function evaluations spend on $A_i$ \emph{after} the potential next phase is minimal among all instances, breaking ties arbitrarily.

Formally, let $\phasel(N_{\max}, i, \phi)$ denote the number of function evaluations in phase $\phi$ of the dynamic \NSGAtwo variant with parameters $N_{\max}$ and $\tau=2^i$. Then, at the beginning of each iteration of the \textbf{while}-loop, $r_i$ equals the number of function evaluations that instance $A_i$ received.
Formally, for the dynamic \NSGAtwo and $N_{\phi} = \min \{4 \cdot 2^{\phi}, N_{\max}\}$, we have
\begin{equation}
    \label{eq:phase-length}
    \phasel(N_{\max}, i, \phi) = \ceil{\tfrac{2^i}{N_\phi}} \cdot N_\phi ,
\end{equation}
since we perform exactly~$N_\phi$ function evaluations per iteration of~$A_i$ and we only stop once we reach or surpass the total budget of $\tau = 2^i$ function evaluations.
For the dynamic \NSGAtwo with longer initial phase, $\phasel$ is the same, except
\begin{equation}
    \label{eq:phase-length-plus}
    \phasel(N_{\max}, i, 0) = \ceil{\log(N_{\max}/4)} \cdot 2^i.
\end{equation}

\begin{algorithm2e}[t]%
    \caption{Concurrent optimization for the dynamic \NSGAtwo and the dynamic \NSGAtwo with longer initial phase, called~$A^\star$, given parameter $N_{\max} \in \N_{> 4}$. Read \cref{line:initialization} as an initialization for whenever these values are used for the first time, not as infinite loop. Recalling \cref{eq:phase-length,eq:phase-length-plus}, \cref{line:min-search} searches with increasing $j \in \N$ and stops once~$2^j$ is larger than all preceding values.}
    \label{alg:concurrent_opt}
    \lFor{$i\in \N_{\ge2}$}{$r_i \coloneqq 0; \quad \phi_i \coloneqq 0$}
    \label{line:initialization}
    \While{no instance meets termination criterion}{%
         $i \coloneqq \arg\min_{j\in \N} \{r_j + \phasel(N_{\max}, j, \phi_j)\}$\;
        \label{line:min-search}
        \lIf{$\phi_i = 0$}{initialize an~$A^\star$ instance $A_i$ with~$N_{\max}$ as in the input and with $\tau = 2^i$}
         Run the next phase of $A_i$\;
        $\phi_i \coloneqq \phi_i +1$\;
        $r_i \coloneqq r_i + \phasel(N_{\max}, i, \phi_i)$\;
    }
\end{algorithm2e}%

Our main runtime result for \Cref{alg:concurrent_opt} is \Cref{thm:concurrent_omm_ojzj}, which shows essentially that on both \omm and \ojzjk, both dynamic \NSGAtwo variants are slower by only a factor of~$O(\log n)$ and $O(k\log n)$, respectively, when compared to an optimal value of~$\tau$ in \Cref{thm:uniform_growth,thm:uniform_growth_ojzj,thm:omm-extended-first-phase,thm:ojzj-extended-first-phase}.
This small slow-down still results in runtimes that are far faster than the typical $O(n^2 \log n)$ and $O(n^{k+1})$ runtimes of many MOEAs on \omm and \ojzjk, respectively, as discussed before.

\begin{theorem}\label{thm:concurrent_omm_ojzj}
    Consider Algorithm~\ref{alg:concurrent_opt} optimizing \omm or \ojzjk with $k \in [2 .. \floor{\frac{n}{2}}]$ and $N_{\max} \ge n+4$.
    Let $F_\omm$ and $F_\ojzjk$ denote the total number of function evaluations until one of the instances covers the Pareto front, respectively.

    If the instances run the dynamic \NSGAtwo (\Cref{alg:nsga-ii-regular-growth}),
    \begin{align*}
        E[F_\omm]   & = O(n\log^3 n) \text{ and }
        E[F_\ojzjk] & = O(k n^k \log^3 n).
    \end{align*}

    If they run the dynamic \NSGAtwo with longer initial phase,
    \begin{align*}
        E[F_\omm]   & = O(n \log(\Nmax) \log (n \log \Nmax)),    \\
        E[F_\ojzjk] & = O(k n^k\log(\Nmax) \log (n \log \Nmax)).
    \end{align*}
    These four bounds also hold with probability $1-O(1/n)$.
\end{theorem}

In contrast to our results in the previous sections, the bounds in \Cref{thm:concurrent_omm_ojzj} hold additionally in expectation, as multiple runs start with a small population size and thus result in independent trials of covering the Pareto front once the population is sufficiently large.
Moreover, as in the case for the dynamic \NSGAtwo, \Cref{alg:concurrent_opt} with the dynamic \NSGAtwo also effectively eliminates~$N_{\max}$ as a parameter, as its runtime is independent of~$N_{\max}$, once sufficiently large.
Hence, choosing arbitrarily large values for~$N_{\max}$ results in an essentially parameterless algorithm, which is a remarkable result, in particular since this parameterless version still exhibits a very fast runtime on \omm and \ojzjk.

\paragraph{Theoretical analysis.}

The balancing of the function evaluations across the instances in Algorithm~\ref{alg:concurrent_opt} allows us to bound the total number of function evaluations from above by a function of the number of evaluations performed within any given instance~$A_i$, as captured in the following lemma.

\begin{lemma}\label{lem:concurrency}
    Let $i\in \N_{\ge2}$ and $r\in \N$. Consider some point in the execution of Algorithm~\ref{alg:concurrent_opt} on any objective function with $N_{\max} \in \N_{> 4}$ such that at least phase 0 of $A_i$ has been conducted and~$A_i$ received at most $r$ function evaluations.
    Then the total number of function evaluations in Algorithm~\ref{alg:concurrent_opt} so far is at most $2r\log(2r)$.
\end{lemma}

The proof of Lemma~\ref{lem:concurrency} uses that in order to start a new instance $i \in \N_{\geq 2}$, its run time budget is larger than the cost of running any of the roughly~$\log i$ smaller instances.

The proof of \Cref{thm:concurrent_omm_ojzj} uses bounds from \Cref{thm:uniform_growth,thm:uniform_growth_ojzj,thm:omm-extended-first-phase,thm:ojzj-extended-first-phase} for optimal choices of~$\tau$.
Then, Lemma~\ref{lem:concurrency} bounds the number of function evaluations in all other instances.
The expected runtimes in \Cref{thm:concurrent_omm_ojzj} follow since the runtimes in \Cref{thm:uniform_growth} hold with high probability for sufficiently large~$\tau$.

\section{Conclusion}

We introduced the \emph{dynamic \NSGAtwo} (\Cref{alg:nsga-ii-regular-growth}), which runs the classic \NSGAtwo with a periodically increasing population size, based on a periodicity~$\tau$ determined by the user.
We proved that the dynamic \NSGAtwo has a superior runtime performance in comparison to typical MOEAs for the \omm (\Cref{thm:uniform_growth}) and \ojzjk (\Cref{thm:uniform_growth_ojzj}) benchmark for a large range of~$\tau$.
For optimal parameter choices, the dynamic \NSGAtwo even obtains speed-ups of $\widetilde{O}(n)$.

Moreover, we showed that if the initial phase of the dynamic \NSGAtwo is slightly longer than the other phases, the optimal performance improves by a factor of $O(\log n)$, both on \omm (\Cref{thm:omm-extended-first-phase}) and \ojzjk (\Cref{thm:ojzj-extended-first-phase}).

Last, in order to remove the need to choose the doubling parameter~$\tau$ adequately, we introduced a scheme that runs the dynamic \NSGAtwo variants concurrently (\Cref{alg:concurrent_opt}).
We showed that this scheme increases the total runtime performance by only a logarithmic factor of the original runtime (\Cref{thm:concurrent_omm_ojzj}), while eliminating the parameter~$\tau$.

For future research, it is interesting to see whether the dynamic \NSGAtwo maintains its advantage in other popular scenarios, such as LOTZ \cite{LaumannsTZ04}.

\section*{Acknowledgments}
Simon Wietheger acknowledges support by the Austrian Science Fund (FWF) [10.55776/Y1329].
Martin~S. Krejca is supported by the French National Agency for Research (ANR) via the JCJC grant titled \emph{MultiEDA} (ANR-25-CE23-2418-01).
The research of Benjamin Doerr and Martin~S. Krejca benefited from the support of the FMJH Program PGMO.

\bibliographystyle{named}
\bibliography{alles_ea_master,ich_master,refs}

    \clearpage
    \appendix

     \section{Mathematical Tools}

    In our runtime analysis, we make use of the following statements.
    They all make use of empty intervals (defined in the preliminaries) and the classic, static \NSGAtwo with current crowding distance (\Cref{alg:nsga-ii-classic}; explained in the algorithm section).
    The statements are originally longer, but we only mention those parts that are relevant to our work.

    The following statement shows that once the parent population of the static \NSGAtwo with current crowding distance contains~$0^n$ and~$1^n$,  the size of each empty interval cannot increase, and it decreases with constant probability. Apart from the trivial extension to more general offspring generation procedures, this result is from~\cite{ZhengD25approx}, more precisely, from the proofs of Lemmas~$13$ and~$14$ in the full version \cite{ZhengD22geccoarxiv} of that work.

    \begin{lemma}[minimally adjusted from~{\protect\cite{ZhengD25approx}}]
        \label{lem:known_progress_intervals}
        Consider the \NSGAtwo (\Cref{alg:nsga-ii-classic}), with survival based on the current crowding distance, and with arbitrary population size $N \ge 4$. Assume that in each iteration, each parent with probability at least $q\in (0, 1]$ is chosen to generate an offspring via bit-wise mutation.
        Consider some iteration $t\in \N_0$ of this algorithm optimizing \omm. Assume that $\set{0^n,1^n}\subseteq P_t$. Let $I_i^{(t)}$ and $I_i^{(t+1)}$ be the $i$-th empty interval in population $P_t$ and $P_{t+1}$, respectively.
        Then for all $i\in [n]$, we have
        \begin{itemize}
            \item $\set{0^n,1^n}\in P_{t+1}$ with probability one,
            \item $|I_i^{(t+1)}| \le \max\set{|I_i^{(t)}|, \floor{\frac{2n}{N-3}}}$ with probability one, and
            \item $|I_i^{(t+1)}| \le \max\set{|I_i^{(t)}|-1, \floor{\frac{2n}{N-3}}, 1}$ with probability at least $\frac{q}{2e}$.
        \end{itemize}
    \end{lemma}

    Moreover, we make use of the concentration of sums of independent random variables.

    \begin{theorem}[{\protect\cite[Theorem~$1.10.5$]{Doerr20bookchapter}}]
        \label{thm:prelims:chernoff}
        Let $n \in \N$, and let $(X_i)_{i \in [n]}$ be independent random variables taking values in~$[0,1]$.
        Then for $X \coloneqq \sum_{i \in [n]} X_i$ and all $\delta \in [0,1]$,
        \begin{equation*}
            \Pr[X \leq (1 - \delta) E[X]] \leq \exp(- \tfrac{\delta^2 E[X]}{2}).
        \end{equation*}
    \end{theorem}

    \begin{theorem}[{\protect\cite[Theorem~$1.10.35$]{Doerr20bookchapter}}]
        \label{thm:prelims:sum-of-geometric-rvs}
        Let $n \in \N$, and let $(X_i)_{i \in [n]}$ be independent geometric random variables with respective success probabilities $(p_i)_{i \in [n]}$.
        Assume that there is a $C \in (0, 1]$ such that for all $i \in [n]$, we have $p_i \geq C \frac{i}{n}$.
        Then for $X \coloneqq \sum_{i \in [n]} X_i$ and for all $\delta \in \R_{\geq 0}$,
        \begin{equation*}
            \Pr[X \geq (1 + \delta) \tfrac{1}{C} n \ln n] \leq n^{-\delta} .
        \end{equation*}
    \end{theorem}

    \section{Runtime Guarantees for \omm}

    We present the proofs that are missing in the main paper.

    \subsection{Proof of \Cref{lem:nsgaII_omm_outwards}}

    For each iteration $t$, let $\ell_t$ and $r_t$ be individuals in the parent population $P_t$ with the minimum and maximum number of ones, respectively.
    Then, for each iteration $t$, we have $\ones{\ell_{t+1}} \le \ones{\ell_t}$ and $\ones{r_{t+1}} \ge \ones{r_t}$ as follows.
    When optimizing \omm, no solution strictly dominates another, so the parent population $P_{t+1}$ is selected from the combined population $R_{t}$ solely based on the crowding distance of the individuals.
    Note that in $R_{t}$ at most 4 individuals have infinite crowding distance and that at least one individual with maximum and minimum number of ones in $R_t$ receives infinite crowding distance, respectively.
    Hence, $\ones{\ell_{t}}$ is non-increasing and $\ones{r_{t}}$ is non-decreasing for increasing $t$.
    Furthermore, if the mutation of $r_t$ in iteration $t$ flips a zero and nothing else, then $\ones{r_{t+1}} \ge \ones{r_t} +1$.
    Such an event happens with probability at least
    \[p_t\coloneqq \frac{n-\ones{r_t}}{n} \nootherflip \ge \frac{n-\ones{r_t}}{en}.\]

    For $i \in [n]$, let $X_i$ be independent geometric random variables, each with success probability $\frac{i}{en}$, and let $X=\sum_{i=1}^n X_i$.
    Then $X$ stochastically dominates the random variable $T_1$, which denotes the number of iterations until $1^n$ is sampled, and thus a tail bound for $X$ also applies to $T_1$.
    By \Cref{thm:prelims:sum-of-geometric-rvs}, for all $\delta \in \R_{\geq 0}$, we have
    \begin{align*}
        \Pr[X\ge (1+\delta)en \ln n]
         & \le n^{-\delta}.
    \end{align*}
    The case for $0^n$ is identical by symmetry.
    Thus, for $\delta~=~1$, we obtain that the probability to sample both $1^n$ and $0^n$ within $2en\ln n$ iterations is at least $(1-\frac{1}{n})^2 \ge 1-\frac{2}{n}$.

    \subsection{Proof of \Cref{lem:nsgaII_omm_inwards}}

    For all $i \in [n]$ and $t'\in \N_0$, let $X_{i,t'}$ denote the length of the $i$-th empty interval in $P_{t'}$.
    By Lemma~\ref{lem:known_progress_intervals} we have for all $t'\ge t$ and all $i\in [n]$ that
    \begin{itemize}
        \item $X_{i,t'+1} \le \max\set{\floor{\frac{2n}{N-3}}, X_{i,t'}}$;
        \item if $X_{i,t'} > \max\set{\floor{\frac{2n}{N-3}},1}$, then $X_{i,t'+1} \le X_{i,t'}-1$ with probability at least $p \coloneqq \frac{q}{2e}$.
    \end{itemize}

    For any $i\in [n]$, consider the event $\mathcal{E}_i$ that $X_{i,t+\delta} \le m'$.
    Let $Z_1, \ldots, Z_{\delta}$ be independent Bernoulli variables with success probability $p$ and let $Z=\sum_{j=1}^{\delta} Z_j$. Then
    \[
        \Pr[\mathcal{E}_i] \ge \Pr[Z \ge X_{i,t}-m'] \ge \Pr[Z \ge m-m'].
    \]
    As $E[Z]=p\delta \ge 2(m-m')$, we can use Theorem~\ref{thm:prelims:chernoff} with $\delta = \frac{1}{2}$ to obtain
    \begin{align*}
        \Pr\left[Z \le m-m'\right] & \le \Pr[Z \le (1-\tfrac 12) E[Z]] \\
                                   & \le \exp(-\tfrac{E[Z]}{8})
        \le \exp(-\tfrac{q\delta}{16e}).
    \end{align*}
    By a union bound over all $i\in [n]$, we obtain
    \begin{align*}
        \Pr[\MEI(P_{t+\delta^*}) > m']
         & = \Pr\left[\bigcup\nolimits_{i\in [n]} \neg \mathcal{E}_i\right] \\
         & \le n \exp(-\tfrac{q\delta}{16e}).\qedhere
    \end{align*}

    \subsection{Proof of Theorem~\ref{thm:uniform_growth}}

    Let $d \coloneqq \ceil{\log(N_{\max}/4)}$ be the number of times the population size is increased.
    For each $i\in [d]_0$, let $N_i \coloneqq \min\{4\cdot 2^i, \Nmax\}$ be the size of the population in phase~$i$, and let
    $\delta_i \coloneqq \ceil{\frac{\tau}{N_i}}$
    be the number of iterations spent in phase~$i$ (where we artificially, for the presentation of this proof, end the $d$-th phase after $\delta_d$ iterations). The end times of the phases are recursively defined by $t_0 = \delta_0$ and, for all $i \in [d]$, by $t_i\coloneqq t_{i-1}+\delta_i$.
    Observe that the population size is increased precisely in iterations $(t_i)_{i \in [d - 1]_0}$.

    By construction, we have
    $t_0  = \delta_0 = \ceil{\frac{\tau}{4}} \ge 2en \ln(n) / q$.
    By Lemma~\ref{lem:nsgaII_omm_outwards}, we have $\set{0^n,1^n}\subseteq P_{t_0}$ with probability at least $1-\frac{2}{n}$.
    Assume this event occurs.
    Then we have $\MEI(P_{t_0}) \le n \le \frac{2n}{N_0-3}$ as $N_0 = 4$.

    Let $r \le d$ be minimal such that $N_r \ge n+4$. Note that this implies $N_r \le 2n+6$. For all $i\in [r]$, let $\EE_i$ be the event that $\MEI(P_{t_i}) \le \max\set{\frac{2n}{N_i-3},1}$. We are interested in the event $\EE_r$, which is equivalent to saying that the population witnesses the Pareto front of \omm (since its $\MEI$ is one).
    We have
    \begin{align*}
        \Pr[\EE_r] & \ge \Pr[\EE_1 \wedge \dots \wedge \EE_r]
        = \prod_{i=1}^r \Pr[\EE_i \mid \EE_1 \wedge \dots \wedge \EE_{i-1}].
    \end{align*}
    Let $i \in [r]$. Note that $\Pr[\EE_i \mid \EE_1 \wedge \dots \wedge \EE_{i-1}]$ is the probability that a static current-crowding-distance \NSGAtwo with initial population equal to $P_{t_{i-1}+1}$, which by assumption has an $\MEI$ of at most $\max\set{\tfrac{2n}{N_{i-1}-3},1} = \tfrac{2n}{N_{i-1}-3}$,
    after $\delta_i$ iterations has reached a population with $\MEI$ at most  $\max\set{\frac{2n}{N_i-3},1}$.

    We estimate this probability via Lemma~\ref{lem:nsgaII_omm_inwards}, using $m\coloneqq  \lfloor\tfrac{2n}{N_{i-1}-3}\rfloor$, $m'\coloneqq \max\set{\floor{\tfrac{2n}{N_i-3}},1}$, and $\delta = \delta_i$. Estimating $\delta_i = \ceil{\frac{\tau}{N_i}} \ge \frac{64en \ln(n) / q}{N_i}$ and $4e(m-m')/q \le 4em/q \le \frac{8en}{(N_{i-1}-3)q} \le \frac{8en}{(N_{i-1}/4)q} \le \frac{16en}{N_{i} q}$, we see that $\delta_i \ge 4e(m-m')/q$, that is, Lemma~\ref{lem:nsgaII_omm_inwards} is applicable. We thus obtain
    \begin{align*}
        \Pr[\neg \EE_i & \mid \EE_1 \wedge \dots \wedge \EE_{i-1}] \le n\exp(-\tfrac{\delta_i q}{16e})                            \\
                       & \le n \exp(-\tfrac{64en \ln(n)}{16e N_i}) \le n \exp(-\tfrac{64en \ln(n)}{16e \cdot 2^{-r+i+1} N_{r-1}}) \\
                       & \le n \exp(-\tfrac{64en \ln(n)}{16e \cdot 2^{-r+i+1} \cdot  (n+3)}) = O(n^{-2^{r-i+1}+1}).
    \end{align*}
    Consequently, $\Pr[\EE_r] \ge \prod_{i=1}^r (1 - O(n^{-2^{r-i+1}+1})) \ge 1 - \sum_{i=1}^r O(n^{-2^{r-i+1}+1}) \ge 1 - O(1/n)$.

    We estimate the computation times if the event $\EE_r$ holds. The number of iterations spent until the end of phase $r$ is $T = \sum_{i=0}^r \delta_i = \sum_{i=0}^r \ceil{\frac{\tau}{N_i}} \le r+1 + O(\tau) = O(\tau)$.
    The number of function evaluations spent in phase $i$ is $\ceil{\frac{\tau}{N_i}}N_i \le \tau + N_i$. Together with the four evaluations of the initial solutions, the total number of function evaluations up to phase $r$ is $F = 4 + \sum_{i=0}^r (\tau + N_i) = O(\tau\log(n) + N_r) = O(\tau \log n) $.

    For the case that $\tau$ is smaller than what was required so far, we argue as follows. In a similar vein as in the previous paragraph, we see that with probability one, after $O(\tau + \log \Nmax)$ iterations and $O(\tau \log(\Nmax) + \Nmax)$ function evaluations the algorithm has reached the maximal population size $\Nmax$.
    From this point on, it behaves like a static \NSGAtwo with population size $N = \Nmax$. By Lemma~\ref{lem:nsgaII_omm_outwards}, with probability $1 - O(1/n)$, after $2en \ln(n)/q$ iterations, the two extremal points of the Pareto front are found.
    By Lemma~\ref{lem:nsgaII_omm_inwards}, applied with $m=n$, $m'=1$ and $\delta = 4en/q$, with probability $1-O(1/n)$, the full Pareto front is found after $\delta$ iterations.
    Hence with probability $1-O(1/n)$, the full Pareto front is found after $O(\tau + \log (\Nmax) + 2en \ln(n)/q + 4en/q) = O(\log \Nmax + n \log n)$ iterations and $O(\tau\log(\Nmax) + \Nmax + \Nmax\cdot2en \ln(n)/q) = O(\Nmax n \log n)$ fitness evaluations.
    Since both Lemma~\ref{lem:nsgaII_omm_outwards} and~\ref{lem:nsgaII_omm_inwards} allow for arbitrary initial populations, a restart argument extends these quantities to also hold in expectation.

    \subsection{Proof of \Cref{thm:omm-extended-first-phase}}
    Due to the similarities between the dynamic \NSGAtwo with longer initial phase and the dynamic \NSGAtwo, this proof is similar to the one of \Cref{thm:uniform_growth}.
    Hence, we follow the same structure as that proof.
    If arguments are the same, we mention this but do not restate the calculations.
    The main difference to the proof of \Cref{thm:uniform_growth} is that the first phase of the dynamic \NSGAtwo with longer initial phase is larger by a factor of order $\log N_{\max} = \Omega(\log n)$ than the other phases.
    Due to the choice of~$\tau$, this still results in $\Omega(n \log n)$ function evaluations for the first phase, but since the other phases are shorter, \Cref{lem:nsgaII_omm_inwards} does not provide us a bound with high probability that the MEI is at most~$1$ but rather at most order~$\log n$.
    Hence, we add an additional step to the arguments from the proof of \Cref{thm:uniform_growth} that shows that the size of the MEI quickly reduces to~$1$ afterward.

    We use the same notation as the proof of \Cref{thm:uniform_growth}, with the difference that we now define $\delta_0 = \ceil{d \tau /4}$, recalling that $d = \ceil{\log(\Nmax/4)}$ and $\tau/4 \ge 4en/q$.
    As $d \ge \log(\frac{n+4}{4}) \ge \ln(n)/2$ we have $\delta_0 \geq 2en\ln(n)/q$.
    Hence, by \Cref{lem:nsgaII_omm_outwards}, we have $\set{0^n,1^n}\subseteq P_{t_0}$ with probability at least $1-\frac{2}{n}$ and that $\MEI(P_{t_0}) \le \frac{2n}{N_0-3}$.

    For the next steps in the proof of \Cref{thm:uniform_growth}, we define $r \leq d$ to be minimal such that $N_r \geq \frac{n}{4\ln n} + 3$.
    Thus, event~$\EE_r$ states that the population in phase~$r$ has an MEI of at most $\frac{2n}{N_r - 3} \leq 8 \ln n$.
    For all $i \in [r]$, we then apply Lemma~\ref{lem:nsgaII_ojzj_inwards} with $m\coloneqq  \lfloor\tfrac{2n}{N_{i-1}-3}\rfloor$ and $m'\coloneqq \max\set{\floor{\tfrac{2n}{N_i-3}},1}$ in the same manner as in the proof of \Cref{thm:uniform_growth}. For this, note that $\delta_i \geq \frac{16 en}{N_i q}$ and recall that $4e(m - m')/q \leq \frac{16en}{N_i q}$.
   We obtain
    \begin{align*}
        \Pr[ & \neg \EE_i  \mid \EE_1 \wedge \dots \wedge \EE_{i-1}] \le n\exp(-\tfrac{\delta_i q}{16e}) \\
             & \le n \exp(-\tfrac{n}{N_i}) \le n \exp(-\tfrac{n}{2^{-r+i+1} N_{r-1}})                    \\
             & \le n \exp(-\tfrac{n}{ 2^{-r+i+1} \cdot  (n/(4\ln(n))+2)}) = O(n^{-2^{r-i+1}+1}).
    \end{align*}
    Hence, $\Pr[\EE_r] \geq 1 - O(1/n)$.

    Conditional on~$\EE_r$, we bound the number of iterations for the MEI to decrease from at most $8 \ln n$ to~$1$.
    To this end, let $d' \coloneqq \ceil{\log(n+4)/4}$.
    We note that all iterations after $t_{d'}$ have population size at least $n+4$. We apply \Cref{lem:nsgaII_omm_inwards} with $m = 8 \ln n$, $m'=1$, and $\delta = \ceil{32e\ln(n)/q}$, noting that $\delta \geq 4e(m - m')/q$.
    Thus, with probability at least $1-n\exp(-\frac{\delta q}{16e})\ge 1-\frac{1}{n}$, in iteration $t_{d'}+\delta$ the population witnesses the entire Pareto front.

    Conditional on all these events, we estimate the computation times.
    The number of iterations spent until the end of phase~$d'$ plus the additional $\delta$ iterations is
    \begin{align*}
        \delta +\ceil{\frac{\tau d}{N_0}} + \sum_{i = 1}^{d'} \ceil{\frac{\tau}{N_i}}
         & \leq   \delta +\tau \log(N_{\max}) + d' + 1 + O(\tau) \\
         & = O(\tau \log N_{\max}).
    \end{align*}
    The number of function evaluations spent in phase~$0$ is $\ceil{\frac{\tau d}{N_0}} N_0 \leq \tau d + N_0$ and, for each phase $i \in [d']$, it is $\ceil{\frac{\tau}{N_i}}N_i \le \tau + N_i$.
    Hence, including the four evaluations from the initialization, until the end of phase $d'$ we have at most $4 + \tau d + N_0 + \sum_{i = 1}^{d'} (\tau + N_i) = O(\tau d + \tau \log(n) + N_{d'}) = O(\tau d)$ function evaluations.
    As the population size is at most $\Nmax$, the last $\delta=O(\ln n)$ iterations perform at most $O(\Nmax \ln n)$ additional function evaluations.

    For the case of smaller values of~$\tau$, we argue as in the proof of \Cref{thm:uniform_growth}.
   The only difference is that the total amount of function evaluations until the dynamic \NSGAtwo with longer initial phase reaches its maximum population size is larger by a factor less than two, and in particular still in $O(\tau d)$.
    Thus, with probability at least $1 - O(1/n)$, the full Pareto front is found after $O(\tau d + N_{\max} n \log n)$ function evaluations.
    A restart arguments yields the result in expectation.

    \section{Runtime Guarantees for OJZJ}

    The structure of the proof is similar to the one for \omm. We first estimate the time until the extreme points $1^n$ and $0^n$ are sampled.
    \begin{lemma}\label{lem:nsgaII_ojzj_outwards}
        Consider the \NSGAtwo (\Cref{alg:nsga-ii-classic}) with arbitrary initialization,  with survival based on the classic or current crowding distance, and with arbitrary population size $N \ge 4$. Assume that in each iteration, each parent with probability at least $q \in (0, 1]$ is chosen to generate an offspring via bit-wise mutation.
        Let $T$ denote the number of iterations until in this algorithm, when optimizing \ojzjk, the population contains $1^n$ and $0^n$.
        Then $T\le 2en \ln(n)/q + en^k\ln(n)/q$ with probability at least $1-\frac{4}{n}$.
    \end{lemma}
    \begin{proof}
        For each iteration $t$, let $x_t$ be an individual in $P_t$ for which $\jump(x_t)$ is maximal, and let $d_t = \jump(x_t)$.
        As in the proof of Lemma~\ref{lem:nsgaII_omm_outwards}, at least one individual in the combined parent and offspring population with maximal first objective value survives into the next generation. In particular, $d_t$ is non-decreasing.
    
        We compute the probability of increasing $d_t$. If $\jump(x_t) < k$, there are at least $n-\jump(x_t)$ bits such that flipping exactly one of them increases the $\jump$ value. If $k \le \jump(x_t) < n$, then there are $n-\jump(x_t)+k$ such bits.
        In every fixed iteration, the probability that $x_t$ mutates in a way that reduces $d_t$ is hence at least
        \[q\frac{n-\jump(x_t)}{n} \nootherflip \ge q\frac{n-d_t}{en}.\]
        Applying \Cref{thm:prelims:sum-of-geometric-rvs} in the same manner as in Lemma~\ref{lem:nsgaII_omm_outwards} we obtain that for all $t\ge 2en\ln(n)/q$ we have $d_t \ge n$ with probability at least $1-\frac{1}{n}$.

        Note that $d_t \ge n$ implies that there is an individual $x_t\in P_t$ such that either $x_t = 1^n$ or $x_t$ has precisely $k$ bits of value $0$. Thus, in every iteration $t'\ge t$ the probability that the combined parent and offspring population contains the bitstring $1^n$ (either because it already exists or by flipping  precisely all 0-bits in a mutation of $x_t$) is at least $q\frac{1}{n^k} \nootherflip \ge \frac{q}{en^k}$.
        After $\ceil{en^k\ln(n)/q}$ iterations, the probability to not have sampled $1^n$ is at most $(1-\frac{q}{en^k})^{en^k\ln(n)/q} \le (\frac{1}{e})^{\ln(n)} = \frac{1}{n}$.

        Summing up the required iterations for the two steps and noting that by symmetry the same arguments apply to finding~$0^n$, we see that both $1^n$ and $0^n$ are sampled after at most
        $2en\ln(n)/q+en^k\ln(n)/q$
        iterations with probability at least $1-\frac{4}{n}$.
    \end{proof}

    To analyze how the population takes over the Pareto front, that is, to prove a result comparable to \Cref{lem:nsgaII_omm_inwards},
    we adjust the notion of empty intervals to account for the gaps on the Pareto front of \ojzjk. We note that so far all mathematical runtime analyses considering the approximation of the Pareto front considered only the \omm benchmark.

    Let~$P$ be a population of individuals containing at least~$0^n$ and~$1^n$.
    Let $P_* \subseteq P$ be the subset of individuals on the Pareto front.
    For all $i \in [n-2k+2]$, the \emph{$i$-th empty interval~$I_i$} (of~$P$ for~\ojzjk) is defined as $
        I_i = $
    \begin{align*}
        [\max & \set{\ojzjk_1(x)\mid x\in P_* \land \ojzjk_1(x)\le 2k+i-1.5}.. \\
        \min  & \set{\ojzjk_1(x)\mid x\in P_* \land \ojzjk_1(x)\ge 2k+i-1.5}].
    \end{align*}
 
    We observe that a population $P$ covers the Pareto front if and only if the two outmost intervals have length~$k$
     and all remaining intervals have length~$1$. To allow for a uniform treatment of these two types of empty intervals, in the following we do not regard the length of empty intervals, but by how far they exceed their minimum possible length. We therefore define the \emph{excess} of the $i$th empty interval by
    \begin{align*}
        \EIE(i) & \coloneq |I_i| - \begin{cases}
                                       k, \text{ if } i \in \{1, n-2k+2\} \\
                                       1, \text{ else,}
                                   \end{cases}
     \end{align*}
     and the \emph{maximum empty interval excess} of $P$ by $\MEIE(P) = \max \set{\EIE(I_i) \mid i \in [n-2k+2]}$. With this definition, a population $P$ covers the Pareto front if and only if $\MEIE(P)=0$.
  
    \begin{lemma}
        \label{lem:known_progress_intervals_ojzj}
        Consider the \NSGAtwo (\Cref{alg:nsga-ii-classic}), with survival based on the current crowding distance, and with arbitrary population size $N \ge 4$. Assume that in each iteration, each parent with probability at least $q\in (0, 1]$ is chosen to generate an offspring via bit-wise mutation.
        Consider some iteration $t\in \N_0$ of this algorithm optimizing \ojzjk. Assume that $\set{0^n,1^n}\subseteq P_t$.
        Let $i \in [n-2k+2]$ and let $I_i^{(t)}$ and $I_i^{(t+1)}$ be the $i$-th empty interval in population $P_t$ and $P_{t+1}$, respectively. Then with probability one
        \begin{itemize}
            \item $\set{0^n,1^n}\in P_{t+1}$,
            \item $\EIE(I_i^{(t+1)}) \le \max\set{\EIE(I_i^{(t)}), \frac{2n}{N-3}-1}$.
        \end{itemize}
        If $N > 4$ and there is an $x\in P_t$ with $k \le \ones{x} \le n-k$, then
        \begin{itemize}
            \item with probability one, there is $x'\in P_{t+1}$ such that $k \le \ones{x'} \le n-k$;
            \item if $\EIE(I_i^{(t)}) \ge \frac{2n}{N-3}$, then with probability at least~$\frac{q}{en}$, we have $\EIE(I_i^{(t+1)}) \le           \EIE(I_i^{(t)})-1
                  $.
        \end{itemize}
    \end{lemma}
    \begin{proof}
        As argued for Lemmas~\ref{lem:known_progress_intervals}~and~\ref{lem:nsgaII_ojzj_outwards},
        the reason why $0^n$ and $1^n$ survive is that for both extreme points on the Pareto front, there are individuals with infinite crowding distance in the population.

        Note that every objective value not on the Pareto front is strictly dominated by one of $\ojzjk(1^n)$ and $\ojzjk(0^n)$. Thus, the set of individuals on the the first non-dominated front of the combined parent and offspring population is precisely the set of Pareto-optimal individuals in that combined population.

        We prove the second item. As all individuals on the first non-dominated are Pareto optimal, the value of each of their objectives lies between $k$ and $n+k$. With this at hand, the same argumentation as for \cite[Lemma~13]{ZhengD25approx} applies, so all individuals from the Pareto front whose current crowding distance is at least $\frac{4}{N-3}$ at some point in the selection step survive into the next generation.
        Keeping in mind that we are only interested in empty intervals between Pareto optimal individuals and crowding distance on the first non-dominated front, which only consists of such individuals, the size of an empty interval created upon the removal of an individual with current crowding distance $d$ is at most $dn/2$.
        Thus, the size of an interval $I$ does not increase beyond $\frac{n}{2}\cdot \frac{4}{N-3} = \frac{2n}{N-3}$. The statement follows as $\EIE(I) \le |I|-1$.

        Assume now that $N>4$ and that there is an individual on the inner part of the Pareto front. To prove the third item, suppose the selection step has already removed all such individuals except for one, which we call $x'$. 
        All remaining individuals on the first non-dominated front thus are equal to $1^n$ or $0^n$. At most four out of these have infinite current crowding distance. The current crowding distance of the remaining ones is strictly less than that of $x'$, since their crowding distance contributions on the one side not facing $x'$ are zero. Hence such an individual is removed before $x'$. As $N > 4$, this ensures that $x'$ is selected into the next generation $P_{t+1}$.

        Regarding the last statement, we allow us to be rather pessimistic about the progress, since the running time will in any case be dominated by sampling the individuals $0^n$ and $1^n$. It is straight-forward to see that mutating one of the endpoints of an empty interval such that the empty interval is strictly smaller in the combined parent and offspring population (by flipping precisely one useful bit) has probability at least $q \cdot \frac{1}{n}\cdot \nootherflip \ge \frac{q}{en}$.\footnote{A more careful analysis would reveal that for all inner intervals, progress is made with constant probability, similar to \omm. Intervals where one endpoint is not on the inner Pareto front, however, are more complicated, making progress here resemble the progress towards first finding $1^n$ and $0^n$ in the \omm benchmark.
        }
    \end{proof}

    We show the following analogue of Lemma~\ref{lem:nsgaII_omm_inwards} for \ojzjk, with essentially the same arguments as there.

    \begin{lemma}\label{lem:nsgaII_ojzj_inwards}
        Consider a run of the \NSGAtwo (\Cref{alg:nsga-ii-classic}) with population size $N > 4$,
        generating the next population in any way that has the property that any parent has a probability of at least $q \in (0, 1]$ to be selected a parent of a bit-wise mutation, and using the current crowding distance for the selection of the next population.
        Consider some iteration $t\in \N_0$ of this algorithm optimizing \ojzjk. Let $P_t$ be arbitrary except that $\set{0^n,1^n}\subseteq P_t$ and there is $x\in P_t$ with $k \le \ones{x} \le n-k$.
        Let $m,m'\in \N$ such that $\floor{\frac{2n}{N-3}}-1 \le m'$ and $\MEIE(P_t) \le m$.
        Let $\delta \in \N$ such that $\delta \ge  2en(m-m')/q$. Then $\Pr[\MEIE(P_{t+\delta}) \le m'] \ge 1 - n \exp(-\frac{\delta q}{8en})$.
    \end{lemma}

    \begin{proof}
        For all $i \in [n-2k+2]$ and $t'\in \N_0$, let $X_{i,t'}$ denote the excess of the $i$-th empty interval in $P_{t'}$.
        By Lemma~\ref{lem:known_progress_intervals_ojzj} we have for all $t'\ge t$ and all $i\in [n-2k+2]$ that
        \begin{itemize}
            \item $X_{i,t'+1} \le \max\{X_{i,t'},\floor{\frac{2n}{N-3}}-1\}$;
            \item if $X_{i,t'} > \max\{0,\floor{\frac{2n}{N-3}}-1\}$,
                  then $X_{i,t'+1} \le X_{i,t'}-1$ with probability at least $p \coloneqq \frac{q}{en}$.
        \end{itemize}

        For any $i\in [n-2k+2]$, consider the event $\mathcal{E}_i$ that $X_{i,t+\delta} \le m'$.
        Let $Z_1, \ldots, Z_{\delta}$ be independent Bernoulli variables with success probability $p$ and let $Z=\sum_{j=1}^{\delta} Z_j$. Then
        \[
            \Pr[\mathcal{E}_i] \ge \Pr[Z \ge X_{i,t}-m'] \ge \Pr[Z \ge m-m'].
        \]
        As $E[Z]=p\delta \ge 2(m-m')$, we can use Theorem~\ref{thm:prelims:chernoff} with $\delta = \frac{1}{2}$ to obtain
        \begin{align*}
            \Pr\left[Z \le m-m'\right] & \le \Pr[Z \le (1-\tfrac 12) E[Z]]     \\
                                       & \le \exp\left(-\tfrac{E[Z]}{8}\right)
            \le \exp(-\tfrac{q\delta}{8en}).
        \end{align*}
        By a union bound over all $i\in [n-2k+2]$, we obtain
        \begin{align*}
            \Pr[\MEI(P_{t+\delta^*}) > m']
             & \le \Pr\left[\bigcup\nolimits_{i\in [n-2k+2]} \neg \mathcal{E}_i\right] \\
             & \le n \exp(-\tfrac{q\delta}{8en}).\qedhere
        \end{align*}
    \end{proof}

    We are now ready to estimate the optimization time of the dynamic \NSGAtwo on \ojzjk. The theorem below, for optimal settings of the parameters, gives a runtime of $O(n^k\log^2 n)$ fitness evaluations. This runtime is by roughly a factor of $n$ better than the $\Theta(n^{k+1})$ runtime of the static \NSGAtwo \cite{DoerrQ23tec,DoerrQ23LB}
    and is very close to the trivial lower bound of $\Omega(n^k)$, stemming from the bottleneck of finding the extreme points $1^n$ and $0^n$.

    The two additional $\log(n)$-factors in our bound arise as follows. One additional $\log n$ factor is required to have  a sufficiently high probability of finding the extremal points of the Pareto front (a bound on the expected time is not sufficient here). The other $\log n$ factor is due to all phases having the same length and, similar to the case of \omm, is avoided when using the other variant of our algorithm, see Theorem~\ref{thm:ojzj-extended-first-phase}.

    \subsection{Proof of \Cref{thm:uniform_growth_ojzj}}

    Just like in the proof of Theorem~\ref{thm:uniform_growth}, let $d \coloneqq \ceil{\log(N_{\max}/4)}$ be the number of times the population size is increased, let $N_i \coloneqq \min\{4\cdot 2^i, \Nmax\}$ be the size of the population in phase $i \in [d]_0$, let $\delta_i \coloneqq \ceil{\frac{\tau}{N_i}}$
    be the number of iterations spent in phase $i \in [d]_0$, and let $t_0 = \delta_0$ and $t_i\coloneqq t_{i-1}+\delta_i$ define the end times of the phases.

    By construction, we have
    $t_0  = \delta_0 = \ceil{\frac{\tau}{4}} \ge 2en \ln(n) / q + en^k\ln(n)/q$.
    By Lemma~\ref{lem:nsgaII_ojzj_outwards}, we have $\set{0^n,1^n}\subseteq P_{t_0}$ with probability at least $1-\frac{4}{n}$.
    Assume this event occurs.

    Crossing the valley of low fitness starting from the extreme individuals $1^n$ and $0^n$ is easier than from the other side, because flipping \emph{any} $k$ bits is sufficient. Thus, with the same reasoning as used for Lemma~\ref{lem:nsgaII_ojzj_outwards},  $e\frac{n^k}{\binom{n}{k}}\ln(n)/q \eqqcolon \delta'$ iterations suffice to sample at least one individual on the inner front with probability at least $1-\frac{1}{n}$.
    Assume that this event occurs.

    From this point on, progress on closing the gaps in the Pareto front can be analyzed the same way as we did in the proof of Theorem~\ref{thm:uniform_growth}. While the probability of progress in Lemmas~\ref{lem:known_progress_intervals_ojzj}~and~\ref{lem:nsgaII_ojzj_inwards} is smaller by a factor of $n/2$ compared to the guarantee given by Lemma~\ref{lem:nsgaII_omm_inwards}, this is compensated by the larger phase length $\tau$ --- note that here the minimum permissible value for $\tau$, which is at least $32en^2\ln(n)/q$, exceeds the one in Theorem~\ref{thm:uniform_growth} by a factor of at least $n/2$.
    For phase~1, note that the remaining iterations after sampling an individual on the Pareto front still suffice to make the required progress since its length $\delta_1 - \delta'$ equals
    \begin{align*}
        \ceil{\tfrac{\tau}{8}}-\delta'
         & \ge \tfrac{e}{2}n^k\ln(n)/q + 4en^2\ln(n)/q - en^{k-1}\ln(n)/q \\
         & \ge 4en^2\ln(n)/q +1
        \ge \tfrac{n}{2} \cdot \delta_1^{\omm},
    \end{align*}
    using $k< n, n \ge 3$ and where $\delta_1^{\omm} = \ceil{8en\ln(n)/q}$ is the minimum permissible number of iterations spent in phase 1 in the proof for \omm.

    We thus obtain that with probability at least $1 - O(1/n)$ we have $T = O(\tau)$ and $F = O(\tau \log n)$.

    Concerning smaller values of $\tau$, we argue similarly as in the \omm case.
    The maximum population size $\Nmax$ is reached after $O(\tau + \log \Nmax)$ iterations and $O(\tau \log(\Nmax) + \Nmax)$ function evaluations. From this point on, the algorithm behaves exactly like the static \NSGAtwo with population size $\Nmax$.

    We can estimate the time then required to cover the Pareto front using, for example, a meta-theorem on MOEAs \cite{WiethegerD24arxiv}. We note that the there mentioned \emph{monotonicity property} is satisfied for the bi-objective \ojzjk by similar arguments as used for the proof of Lemma~\ref{lem:known_progress_intervals_ojzj}.
    We obtain that the number of function evaluations to cover the Pareto front with population size $\Nmax$ is at most $O(\Nmax n^k \log n)$ with probability at least $1-O(1/n)$ and at most $O(\Nmax n^k)$ in expectation.
    This matches established bounds for the static \NSGAtwo not using the current crowding distance \cite{DoerrQ23tec,DoerrQ23LB}. \qedhere

    \subsection{Proof of \Cref{thm:ojzj-extended-first-phase}}

    We base this analysis on the proof of \Cref{thm:uniform_growth_ojzj} and show that despite the weaker bound on $\tau$ here all phases are still sufficiently long for the arguments there to apply.
    As discussed in the proof of \Cref{thm:omm-extended-first-phase}, the length of phase~$0$ in the dynamic \NSGAtwo with longer initial phase is larger by a factor of $\log(\Nmax/4)$ compared to the dynamic \NSGAtwo. This compensates for the bound on the $\tau$ here being weaker by a factor of about $\frac{\ln n}{2} \le \log((n+4)/4)\le \log(\Nmax/4)$ for $n\ge 2$.

    For phases~$2$ and later, first assume $k\ge 3$. The length of phase $i$ (with population size $N_i$) is $\ceil{\tau/N_i} \ge 8en^3/(qN_i)$ and thereby larger than required, as described in the other proofs. A similar argument holds for phase~$1$ if $k\ge 3$. If instead $k \le 2$, we have a much better probability of making progress than assumed in Lemmas~\ref{lem:known_progress_intervals_ojzj}~and~\ref{lem:nsgaII_ojzj_inwards}: While we there estimate the probability to decrease the length of an interval with $\frac{q}{en}$, reducing the length of one of the outmost intervals happens with probability at least $\frac{q\binom{n}{2}}{en^k}$ by flipping any $k$ bits in the mutation of $1^n$ or $0^n$. For $k\le 2, n\ge 4$, this quantity is at least $\frac{3q}{8e}$. In any other interval progress happens with probability $\frac{q}{2e}$
    (same as for \omm), so our estimate in Lemmas~\ref{lem:known_progress_intervals_ojzj}~and~\ref{lem:nsgaII_ojzj_inwards} for the case of $k \le 2$ is too weak by a factor of at least $8n/3 \ge \ln n$ (for $n\ge 4)$, which compensates for weaker bound on $\tau$.
    Thereby, all arguments in the proof \Cref{thm:uniform_growth_ojzj} still apply. Accounting for the extended length of the first phase then yields the stated guarantees. \qedhere

    \section{Choosing the Phase Length Automatically}

    \subsection{Proof of Lemma~\ref{lem:concurrency}}

    Consider the beginning of the iteration of the \textbf{while}-loop in which this $r$-th evaluation occurred.
    Observe that, as $\phi_i >~0$, we have $\phasel(N_{\max}, i, \phi_i) \le r_i$, so
    \begin{equation*}
        r_i + \phasel(N_{\max}, i, \phi_i) \le 2r_i.
    \end{equation*}
    Thus, for all $j > \log(2r_i)$, instance~$A_j$ has not been initialized.
    For any $j \le \log(2r_i)$, consider the last iteration in which $A_j$ was selected, if ever.
    Let $r'_i, r'_j, \phi'_i, $ and $\phi'_j$ denote the respective values at the beginning of that iteration.
    Then, at that point, $A_j$ had received $r'_j$ evaluations and $A_i$ had received $r'_i \le r_i$ evaluations.
    As $A_j$ was picked,
    \begin{align*}
        r_j
         & = r'_j + \phasel(N_{\max}, j, \phi_j')    \\
         & \le  r'_i + \phasel(N_{\max}, i, \phi_i')
        \le 2r_i.
    \end{align*}

    Thus, the total number of function evaluations until the $r$-th function evaluation in $A_i$ was conducted is at most $2r\log(2r)$.

    \subsection{Proof of \Cref{thm:concurrent_omm_ojzj}}

    We first discuss the dynamic \NSGAtwo on \omm.
    Let $\tau \coloneqq 64e n \ln (n)/q$.
    By Theorem~\ref{thm:uniform_growth}, for every power of two $\tau'$ with $\tau' \ge \tau$ we have that in a run of Algorithm~\ref{alg:concurrent_opt} the dynamic \NSGAtwo instance $A_{\log \tau'}$ covers the Pareto front after receiving $O(\tau' \log n)$ function evaluations with probability $1-O(1/n)$.
    As for the smallest such $\tau'$ we have $\tau'=O(n\log n)$ and using Lemma~\ref{lem:concurrency}, $F_\omm=O(2\tau'\log(n) \log(2\tau'\log n)) = O(n\log^3 n)$ with probability at least $1-O(1/n)$.

    By Theorem~\ref{thm:omm-extended-first-phase}, using the same arguments for the dynamic \NSGAtwo with longer initial phase with
    $\tau \coloneqq \frac{256}{5}en /q = O(n)$, one of the instances covers the Pareto front after
    receiving
    $O\left(n \log(\Nmax) \log(n \log \Nmax)\right)$
    evaluations
    with probability at least $1-O(1/n)$.

    The same reasoning applies to \ojzjk using the bounds from \cref{thm:uniform_growth_ojzj,thm:ojzj-extended-first-phase}: For the smallest admissible $\tau$ for the dynamic \NSGAtwo we have
    $\tau=O(n^k \log n)$ which yields $O(\tau\log n)$ function evaluations in the critical run and $O(n^k \log^2(n) \log(n^k \log^2 n)) = O(k n^k \log^3 n)$ evaluations in total.
    Similar, for the  dynamic \NSGAtwo with longer initial phase we have $\tau=O(n^k)$ and obtain a total number of
    $O(\tau \log (\Nmax) \log(\tau \log \Nmax))
        = O(n^k \log (\Nmax) \log(n^k \log \Nmax))$
    function evaluations.

    For the expected value, when running \Cref{alg:concurrent_opt} with the dynamic \NSGAtwo on \omm, recall that, for $s\coloneqq \ceil{\log(64 e n \ln (n)/q)}$ and each $s'\in \N_{\ge s}$, instance $A_{s'}$ covers the Pareto front in $O(\log(n)2^{s'})$ function evaluations with probability at least $1-O(1/n)$.
    This event, for any fixed $s'$, corresponds to $O(2^{s'}\log(n)(s'+\log\log n))=O(2^{s'} s'\log n)$ function evaluations in total using Lemma~\ref{lem:concurrency}.
    For the ease of argument, consider a modified version of Algorithm~\ref{alg:concurrent_opt} that does not stop unless some instance $A_{s'}$ with $s'\ge s$ covers the Pareto front within the first $O(2^{s'} s'\log n)$ function evaluations. Clearly, the expected number of evaluations until this modified version terminates exceeds the one for the actual algorithm.
    Hence, the expected number of function evaluations is at most
    \begin{align*}
        O\left(\sum_{i=0}^\infty (O(1/n))^i \cdot  2^{s+i}(s+i)\log n \right) \\
        = O(2^s s\log n)
        = O(n\log^3 n).
    \end{align*}

    The same reasoning applies to the dynamic \NSGAtwo on \ojzjk, except that now  $s\coloneqq \ceil{\log(4 e (n^k + 8n^2) \ln (n)/q)}$. In this case, the  total expected number of evaluations is at most
    $O(2^s s\log n) = O(k n^k\log^3 n)$.

    For the dynamic \NSGAtwo with longer initial phase on \omm with $s\coloneqq \ceil{\log(\frac{256}{5} e n/q)}$, the desirable event for any instance $A_{s'}$ with $s'\ge s$ results in a total of
    $O(2^{s'}\log(\Nmax)\log(2^{s'}\log \Nmax))$ evaluations. The total expected number of evaluations is thus
    \begin{align*}
        O\bigg(\sum_{i=0}^\infty & (O(1/n))^i \cdot
        2^{s+i}\log(\Nmax)\log(2^{s+i}\log \Nmax) \bigg)                               \\
                                 & = O\left(2^{s}\log(\Nmax)\log(2^s\log \Nmax)\right) \\
                                 & = O(n\log(\Nmax)\log(n\log \Nmax)).
    \end{align*}
    Last, with $s \coloneqq \ceil{\log(4e(2n^k+8n^2)/q)}$,
    the expected number of evaluations for the dynamic \NSGAtwo with longer initial phase on \ojzjk is
    \begin{align*}
        O & \left(2^{s}\log(\Nmax)\log(2^s\log \Nmax)\right) \\
          & = O(n^k\log(\Nmax)\log(n^k\log \Nmax)).
    \end{align*}

\end{document}